\documentclass[twoside,leqno,twocolumn]{article}
\usepackage{ltexpprt}

\usepackage{cite}
\usepackage{amsmath,amssymb,amsfonts}
\usepackage{algorithmic}
\usepackage{graphicx}
\usepackage{textcomp}
\usepackage{xcolor}
\def\BibTeX{{\rm B\kern-.05em{\sc i\kern-.025em b}\kern-.08em
		T\kern-.1667em\lower.7ex\hbox{E}\kern-.125emX}}

\usepackage{color}
\usepackage{xspace}
\usepackage{comment}
\usepackage{microtype}
\usepackage{booktabs} 	
\usepackage{centernot} 	
\usepackage{rotating}	
\usepackage{enumitem}	
\usepackage{tikz}		
\usepackage{rotating}	
\usepackage{bm}			
\usepackage{pgfplots}	
\usepackage{balance}	
\usetikzlibrary{arrows.meta, calc, shapes, arrows, positioning, external}
\usepgfplotslibrary{fillbetween}	
\usepackage{thmtools, thm-restate}
\usepackage[utf8]{inputenc} 
\usepackage[T1]{fontenc}    
\usepackage{hyperref}       
\usepackage{url}            
\usepackage{booktabs}       
\usepackage{amsfonts}       
\usepackage{nicefrac}       
\usepackage{microtype}      
\usepackage{mathptmx}

\newcommand{\citep}{\cite} 
\include{preamble}		

\newif\ifincludepdf
\includepdftrue

\begin{document}
	\title{\Large Discovering Reliable Causal Rules}
	\author{Kailash Budhathoki\footnote{\rule{0pt}{1.2em}Amazon Research T{\"u}bingen (work done prior to joining Amazon)}
		\and  Mario Boley\thanks{Monash University} \and Jilles Vreeken\thanks{CISPA Helmholtz Center for Information Security}}
	\date{}
	
	\maketitle
	\begin{abstract}%
		{\small\baselineskip=9pt We study the problem of deriving policies, or \emph{rules}, that when enacted on a complex system, \emph{cause} a desired outcome. Absent the ability to perform controlled experiments, such rules have to be inferred from past observations of the system's behaviour. This is a challenging problem for two reasons: First, observational effects are often unrepresentative of the underlying causal effect because they are skewed by the presence of confounding factors. Second, naive empirical estimations of a rule's effect have a high variance, and, hence, their maximisation can lead to random results.\break
		\indent To address these issues, first we measure the causal effect of a rule from observational data---adjusting for the effect of potential confounders. Importantly, we provide a graphical criteria under which causal rule discovery is possible. Moreover, to discover reliable causal rules from a sample, we propose a conservative and consistent estimator of the causal effect, and derive an efficient and exact algorithm that maximises the estimator. On synthetic data, the proposed estimator converges faster to the ground truth than the naive estimator and recovers relevant causal rules even at small sample sizes. Extensive experiments on a variety of real-world datasets show that the proposed algorithm is efficient and discovers meaningful rules.%
		}%
	\end{abstract}

	\section{Introduction}\label{sec:intro}
	The ultimate goal of meaningful data analysis is to understand how the data was generated, by reasoning in terms of cause and effect. 
	Towards this goal, rule mining~\cite{agrawal:93:associationrules,wrobel:97:sgd,friedman:99:bumphunting,furnkranz:12:rule-book} has been studied extensively over the years. Most rule miners measure the effect of a rule in terms of correlation or dependence. Correlation, however, does not imply causation. As a result, rules that maximise such effect measures are in no way guaranteed to reflect the underlying data-generating process.
	
	The gold standard for establishing the causal relationship between variables is through a controlled experiment, such as a randomized controlled trial (RCT)~\cite{hernan:18:book}. In many cases, however, it is impossible or at the very least impractical to perform an RCT. We hence most often have to infer causal dependencies from observational data, which is data that was collected without full control. In this work, we study discovering causal rules from observational data that maximise causal effect. 
	Though simple to state, this is a very hard task. Not only do we have to cope with an intricate combination of two semantic problems---one statistical and one structural---but in addition the task is also computationally difficult.
	
	The structural problem is often referred to as Simpson's paradox. Even strong and confidently measured effects of a rule might not actually reflect true domain mechanisms, but can be mere artefacts of the effect of other variables. Notably, such confounding effects can not only attenuate or amplify the marginal effect of a rule on the target variable, in the most misleading cases they can even result in sign reversal, i.e. when interpreted naively, the data might indicate a negative effect even though in reality there is a positive effect~\cite[Chap. 6]{pearl:09:book}. For example, a drug might appear to be effective for the treatment of a disease for the overall population. However, if the treatment assignment was affected by sex that also affects the recovery (say males, who recover---regardless of the drug---more often than females, are also more likely to use the drug than females), we may find that the treatment is \emph{not} effective at all to male and female subpopulations.
	
	The statistical problem is the well-known phenomenon of overfitting. This phenomenon results from the high variance of the naive empirical (or ``plug-in'') estimator of causal effect for rules with too small sample sizes for the instances either covered, or excluded by the rule.
	Combined with the maximization task over a usually very large rule language, this variance turns into a strong positive bias that dominates the search and causes essentially random results of either extremely specific or extremely general rules.
	
	Third, the rule space over which we maximise causal effect is exponential in size and does not exhibit structure that is trivially exploited. We therefore need an efficient optimization algorithm. 
	In this paper, we present a theoretically sound approach to discovering causal rules that remedies each of these problems. 
	\begin{enumerate}
		\item To address the structural problem, we propose to measure the causal effect of a rule from observational data. To this end, we control for the effect of a given set of potential confounder variables. In particular, we give a graphical criteria under which it is possible to discover causal rules. While in practice the set of control variables will rarely be complete, i.e., not contain all potential confounders, this approach can rule out specific alternative explanations of findings as well as eliminate misleading observations caused by selected observables that are known to be strong confounders. In fact, this pragmatic approach is usually a necessity caused by sparsity. 
		\item To address the overfitting problem, we propose to measure and optimise the \emph{reliable} effect of a rule. In contrast to the plug-in estimator, we propose a conservative empirical estimate of the population effect, that is not prone to overfitting. Additionally, and in contrast to other known rule optimisation criteria, it is also \emph{consistent}, i.e., with increasing amounts of evidence (data), the measure converges to the actual population effect of a rule.
		\item We develop a practical algorithm for efficiently discovering the top-$k$ strongest reliable causal rules. In particular, we show how the optimisation function can be cast into a branch-and-bound approach based on a computationally efficient and tight optimistic estimator. 
	\end{enumerate}
	
	We support our claims by experiments on both synthetic and real-world datasets as well as by reporting the required computation times on a large set of benchmark datasets.
	
	\section{Related Work}\label{sec:rel}
	\textbf{Association rules.}
	In rule-based classification, the goal is to find a set of rules that optimally predict the target label. Classic approaches include CN2~\cite{lavrac:04:cn2}, and FOIL~\cite{quinlan:95:foil}. In more recent work, the attention shifted from accuracy to optimising more reliable scores, such as area under the curve (AUC)~\cite{furnkranz:05:roc}. 
	
	In association rule mining~\cite{agrawal:93:associationrules}, we can impose hard constraints on the relative occurrence frequency to get reliable rules. In emerging and contrast pattern mining~\cite{dong:99:emerging-patterns,bay:01:contrast}, we can get reliable patterns whose \emph{supports} differ significantly between datasets by performing a statistical hypothesis test.
	Most subgroup discovery~\cite{wrobel:97:sgd} methods optimise a surrogate function based on some null hypothesis test. The resulting objective functions are usually a multiplicative combination of coverage and effect.
	
	All these methods optimise associational effect measures that are based on the observed joint distribution. Thus they capture correlation or dependence between variables. They do not reflect the effect if we were to intervene in the system. \break
	
	\noindent \textbf{Causal rules.} Although much of literature is devoted in mining reliable association rules, a few proposals have been made towards mining causal rules.
	Silverstein et al.~\cite{silverstein:00:causal-assoc-rule} test for pairwise dependence and conditional independence relationships to discover causal associations rules that consist of a univariate antecedent given a univariate control variable. 
	Li et al.~\cite{li:15:causalrule} discover causal rules from observational data given a target by first mining association rules with the target as a consequent, and performing cohort studies per rule. 
	
	Atzmueller \& Puppe~\cite{atzmueller:09:causalsubgroup} propose a semi-automatic approach to discovering causal interactions by mining subgroups using a chosen quality function, inferring a causal network over these, and visually presenting this to the user. 
	Causal falling rule lists~\cite{wang:15:fallingrule} are sequences of ``if-then'' rules over the covariates such that the effect of a specific intervention decreases monotonically down the list from \emph{experimental} data. 
	Shamsinejadbabaki et al.~\cite{blockeel:13:causal-action-mining} discover actions from a partial directed acyclic graph for which the post-intervention probability of \target differs from the observational probability.

	While all these methods have opened the research direction, we still lack a theoretical understanding. Roughly speaking, all these methods propose to condition ``some'' effect measure on ``some'' covariates. In this work, we present a theoretical result showing which covariates to condition upon, under what conditions causal rule discovery is possible, and how an effect measure must be constructed to capture causal effect.
	Overall, despite the importance of the problem, to the best of our knowledge there does not exist a theoretically well-founded, efficient approach to discovering reliable causal rules from observational data.
	
	\section{Reliable Causal Rules}\label{sec:rules}
	We consider a system of discrete random variables with a designated \textbf{target variable} \target and a number of covariates, which we differentiate into \textbf{actionable variables}\footnote{Although an actionable variable (e.g. blood group) may not be directly physically manipulable, a causal model such as a structural equation model~\citep{pearl:09:book} permits us to compute the effect of intervention on such variables.} $\actionables \coloneqq (\attribute_1, \dotsc, \attribute_\ell)$ and \textbf{control variables} $\controls \coloneqq (\control_1, \dotsc, \control_m)$.
	For example, \target might indicate recovery from a disease, \actionables different medications that can be administered to a patient, and \controls might be attributes of patients, such as blood group. Let $\attributespace_j$ denote the domain of $\attribute_j$, and $\controlspace_j$ be that of $\control_j$. As such, the domain of \actionables is the Cartesian product $\actionablesSpace = \attributespace_1 \times \dotsm \times \attributespace_\ell$, and that of \controls is $\controlsspace = \controlspace_1 \times \dotsm \times \controlspace_m$.
	
	We use Pearl’s do-notation~\cite[Chap. 3]{pearl:09:book} $\doo(X \coloneqq x)$, or $\doo(x)$ in short, to represent the \textbf{atomic intervention} on variable $X$ which changes the system by assigning $X$ to a value $x$, keeping everything else in the system fixed. The distribution of \target after the intervention $\doo(x)$ is represented by the \textbf{post-intervention} distribution $\Pr(\target \mid \doo(X \coloneqq x))$. This may not be the same as the \textbf{observed} conditional distribution $\Pr(\target \mid X=x)$. As we observe $\Pr(\target \mid X=x)$ without controlling the system, other variables might have influenced \target, unlike in case of $\Pr(\target \mid \doo(X \coloneqq x))$. Therefore, to capture the underlying data-generating mechanism, we have to use the post-intervention distribution $\Pr(\target \mid \doo(X \coloneqq x))$.
	
	Let $\mathcal{S}$ be the set of all possible vector values of all possible subsets of actionable variables. More formally, we have the following definition:
	\[
	\mathcal{S} = \bigcup\limits_{\actionablesVal \in \mathcal{P}(\{ \actionableSpace_1, \dotsc, \actionableSpace_\ell \})} \actionablesVal \; ,
	\]
	where $\mathcal{P}(\bullet)$ is the powerset function.
	In this work, we are concerned with \textbf{rules} $\drule: \mathcal{S}  \rightarrow \{\true, \false\}$ that for a given value $\actionablesVal \in \mathcal{S}$ evaluate to either true (\true) or false (\false).
	Specifically, we investigate the \textbf{rule language} \drulespace of conjunctions of \textbf{propositions} $\drule \equiv \pi_1 \wedge \dots \wedge \pi_l$ that can be formed from inequality and equality conditions on actionable variables $\attribute_j$s (e.g. $\pi \equiv \text{dosage} \geq 450$). 
	
	Let $\actionablesSubset \subseteq \actionables$ denote the subset of actionable variables, with their joint domain \actionablesSubsetSpace, on which propositions of a rule \drule are defined. Most rule miners measure the effect of a rule using the observed conditional distribution, 
	\[
	\Pr(Y \mid \sigma=\true) = \sum\limits_{\drule(\actionablesVal)=\true} \Pr(\target \mid \attributes=\actionablesVal)\;,
	\]
	which captures the correlation or more generally dependence between the rule and the target. To understand the underlying data-generating mechanism, however, we need post-intervention distributions.
	
	One caveat with rules is that, in general, there are many values \actionablesVal that can satisfy a rule \drule (e.g., $\drule \equiv \attribute_j \leq 3$ is satisfied by $\attribute_j = 3, 2, \dots$). As a result, we have a multitude of atomic interventions to consider (e.g. for $\drule \equiv \attribute_j \leq 3$, we have $\Pr(\target \mid \doo(\attribute_j \coloneqq 3)), \Pr(\target \mid \doo(\attribute_j \coloneqq 2)), \dots$).
	Depending on the atomic intervention we choose, we may get different answers. This ambiguity can be avoided by considering the average of all post-intervention distributions where the probability of each atomic intervention is defined by some \textbf{stochastic policy} $Q_\drule$~\cite[Chap. 4]{pearl:09:book}. In reinforcement learning,
	for instance, a stochastic policy is the conditional probability of an action given some state. Formally, the post-intervention distribution of \target under the stochastic policy $Q_\drule$ is given by
	\begin{align}
	\mathclap{
		\Pr(\target \mid \doo(Q_\drule)) = \sum\limits_{\drule(\actionablesVal)=\true} \Pr(\target \mid \doo(\attributes \coloneqq \actionablesVal)) Q_\drule(\doo(\attributes \coloneqq \actionablesVal)) \; .
	}
	\end{align}
	
	Let $\bar{\drule}$ denote the logical negation of \drule. Our goal is to identify rules \drule that have a high \textbf{causal effect} on a specific \textbf{outcome} \targetval for the target variable \target , which we define as the difference in the post-intervention probabilities of \targetval under the stochastic policies corresponding to $\drule$ and $\bar{\drule}$, i.e.,
	\begin{align}
	\ec = \pmf(\targetval \mid \doo(Q_\drule)) - \pmf(\targetval \mid \doo(Q_{\bar{\drule}})) \; ,
	\end{align}
	where \pmf represents the probability mass function.
	Next we show how to compute the above from observational data, and state the stochastic policy to this end.
	
	\subsection{Causal Effect from Observational Data}\label{subsec:structure}
	In observational data, we have observed conditional distributions $\Pr(\target \mid \attributes=\actionablesVal)$ which may not be the same as post-intervention distributions $\Pr(\target \mid \doo(\attributes \coloneqq \actionablesVal))$. 
	A well-known reason for this discrepancy is the potential presence of \textbf{confounders}, i.e., variables that influence both, our desired intervention variable(s) and the target.
	More generally, to measure the causal effect, we have to eliminate the influence of all \textit{spurious path} in the \textbf{causal graph}, i.e., the directed graph that describes the conditional independences of our random variables (with respect to all post-intervention distributions).
	
	In more detail, when estimating the causal effect of \attribute on \target, any undirected path connecting \target and \attribute that has an incoming edge towards \attribute is a \textbf{spurious path}.
	A node (variable) is a \textbf{collider} on a path if its in-degree is 2, e.g., \control is a collider on the path $\attribute \rightarrow \control \leftarrow \target$.  A spurious path is \textbf{blocked} by a set of nodes \somecontrols, if the path contains a collider that is not in \somecontrols, or a non-collider on the path is in \somecontrols~\cite[Def. 1.2.3]{pearl:09:book}. A set of nodes \somecontrols satisfies the \textbf{back-door criterion} for a set of nodes \attributes and a node \target if it blocks all spurious paths from any \attribute in \attributes to \target, and there is no direct path from any \attribute in \attributes to any \control in \somecontrols~\cite[Def. 3.3.1]{pearl:09:book}. For \attributes and \target, if a set \somecontrols satisfies the back-door criterion, then observational and post-intervention probabilities are equal within each \controlsval stratum of \somecontrols:
	\begin{align}
	\pmf(\targetval \mid \doo(\attributes \coloneqq \actionablesVal), \controlsval) = \pmf(\targetval \mid \actionablesVal, \controlsval)  \label{eq:backdoor} \; ,
	\end{align}
	and averaging the observational probabilities over \somecontrols gives $\pmf(\targetval   \mid \doo(\attributes \coloneqq \actionablesVal))$~\cite[Thm. 3.3.2]{pearl:09:book}.
	
	Therefore, to compute the post-intervention probability of \targetval under the stochastic policy $\policy_\drule$ for a rule \drule, i.e. $\pmf(\targetval \mid \doo(Q_\drule))$, we need a set of variables \somecontrols that satisfy the back-door criterion for actionable variables $\attributes \subseteq \descriptionVars$ and \target. As we consider the rule language \drulespace over all actionable variables \descriptionVars, we require a set of control variables \controls that satisfy the back-door criterion for \textit{all} the actionable variables \actionables.
	This also implies that there are no other spurious paths via potentially unobserved variables \latents.
	In the special case when \controls is empty, \target must not cause any actionable variable $\attribute_j \in \actionables$. We formalise these conditions in the definition below.
	
	\begin{definition}[Admissible Input to Causal Rule Discovery]
		The causal system $(\actionables, \target, \controls)$ of actionable variables, target variable, and control variables is an admissible input to causal rule discovery if the underlying causal graph of the variables satisfy the following:
		\begin{enumerate}[leftmargin=*, label=(\alph*)]
			\setlength\itemsep{0em}
			\item there are no outgoing edges from \target to any \attribute in \actionables,
			\item no outgoing edges from any \attribute in \actionables to any \control in \controls,
			\item no edges between actionable variables \actionables, and
			\item no edges between any unobserved \latent and \attribute in \actionables.
		\end{enumerate}
		\label{def:admissible-input}
	\end{definition}
	
	\tikzset{
		event/.style={draw, circle, fill=black, inner sep=0pt, outer sep=0pt, minimum size=1mm},
		dotsevent/.style={draw=white, circle, fill=white, inner sep=0pt, outer sep=0pt, minimum size=3mm},
		phantom/.style={event, fill=white, draw=white, font=\tiny},
		link/.style={thin, black, -{Latex[length=2mm, width=0.8mm]}, shorten >=0mm},
		dlink/.style={link, densely dashed},
	}
	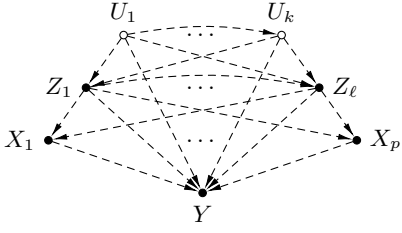
\begin{figure}[tb]
		\centering
		\ifincludepdf
			\includegraphics{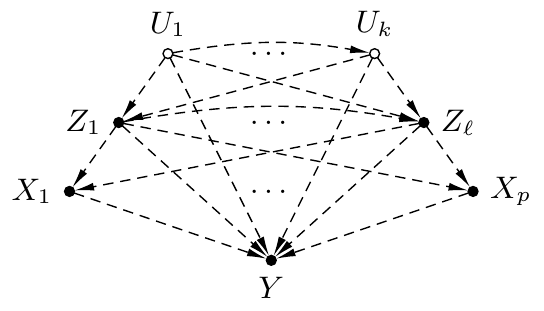}
		\else
			\begin{tikzpicture}[node distance=7mm, font=\small]
			\node [event, label=left:$\control_1$] (z1) {};
			\node [event, right=3cm of z1, label=right:$\control_\ell$] (zl) {};
			\node [dotsevent] (zd) at ($(z1)!0.5!(zl)$) {$\dots$};
			\node [phantom, right=4mm of z1] (zp) {};
			
			\node [phantom, above of=z1] (up1) {};
			\node [phantom, above of=zl] (up2) {};
			\node [event, fill=white, right=4mm of up1, label=above:$\latent_1$] (u1) {};
			\node [event, fill=white, left=4mm of up2, label=above:$\latent_k$] (uk) {};
			\node [dotsevent] (ud) at ($(u1)!0.5!(uk)$) {$\dots$};
			
			\node [phantom, below of=z1] (xp1) {};
			\node [phantom, below of=zl] (xp2) {};
			\node [event, left=4mm of xp1, label=left:$\attribute_1$] (x1) {};
			\node [event, right=4mm of xp2, label=right:$\attribute_p$] (xm) {};
			\node [dotsevent] (xd) at ($(x1)!0.5!(xm)$) {$\dots$};
			
			\node [event, below of=xd, label=below:$\target$] (y) {};
			
			\draw [dlink] (x1) -- (y);
			\draw [dlink] (xm) -- (y);
			
			\draw [dlink] (z1) -- (x1);
			\draw [dlink] (z1) -- (xm);
			\draw [dlink] (zl) -- (x1);
			\draw [dlink] (zl) -- (xm);
			\draw [dlink] (z1) -- (y);
			\draw [dlink] (zl) -- (y);
			
			\draw [dlink] (u1) -- (z1);
			\draw [dlink] (u1) -- (zl);
			\draw [dlink] (uk) -- (z1);
			\draw [dlink] (uk) -- (zl);
			
			\draw [dlink] (u1) -- (y);
			\draw [dlink] (uk) -- (y);
			
			\draw [dlink] (z1) to[out=10, in=170] (zl);
			\draw [dlink] (u1) to[out=10, in=170] (uk);
			
			\end{tikzpicture}
		\fi	
		\caption{A skeleton causal graph of an admissible input to causal rule discovery (see Def.~\ref{def:admissible-input}). A dashed edge from a node $u$ to $v$ indicates that $u$ potentially affects $v$. 
		}
		\label{fig:admissible-input}
	\end{figure}
	
	In Fig.~\ref{fig:admissible-input}, we show a skeleton causal graph of an admissible input to causal discovery. The proposition below shows that the control variables \controls block all spurious paths between any subset of actionable variables $\attributes \in \descriptionVars$ and \target if the input is admissible.
	\begin{restatable}{proposition}{propbackdoor}\label{prop:backdoor}
		Let (\actionables, \target, \controls) be an admissible input to causal rule discovery. Then the control variables \controls block all spurious paths between any subset of actionable variables $\attributes \subseteq \actionables$ and \target.
	\end{restatable}
	\proofApx
	
	Using admissible control variables \controls, we can then compute $\pmf(\targetval \mid \doo(\policy_\drule))$ for any rule \drule from the rule language \drulespace as
	\begin{align}
	\pmf(\targetval \mid \doo(\policy_\drule)) 
	&=\sum\limits_{\drule(\actionablesVal)=\true} \sum\limits_{\controlsvec \in \controlsspace} \pmf(\targetval \mid \attributesval, \controlsval) \pmf(\controlsval) \policy_\drule(\doo(\actionablesVal)) \\
	&=\sum\limits_{\controlsvec \in \controlsspace} \pmf(\controlsval) \sum\limits_{\drule(\actionablesVal)=\true} \pmf(\targetval \mid \attributesval, \controlsval) \policy_\drule(\doo(\actionablesVal)) \; ,
	\end{align}
	where the first expression is obtained by applying the back-door adjustment formula~\cite[Thm. 3.3.2]{pearl:09:book} and the second expression is obtained from the first by exchanging the inner summation with the outer one.
	What is left now is to define the stochastic policy $\policy_\drule$ which in some sense we treated as an oracle so far. The following theorem shows that, with a specific choice of $\policy_\drule$,  we can compute the causal effect of any rule on the target, from observational data, in terms of simple conditional expectations (akin to conditional average treatment effect~\cite{imai:13:heterogenity}).
	
	
	\begin{restatable}{thm}{theoremcausaleffect}\label{theorem:causaleffect}
		Given an admissible input to causal rule discovery, $(\actionables, \target, \controls)$, and a stochastic policy $\policy_\drule(\doo(\actionablesVal)) = \pmf(\attributes=\attributesval \mid \drule=\true, \controls=\controlsval)$, the causal effect of any rule \drule, from the rule language \drulespace, on \target in observational data is given by
		\begin{align}
		\ec = \expectation{\pmf(\targetval \mid \drule, \controls)} - \expectation{\pmf(\targetval \mid \bar{\drule}, \controls)} \; .
		\label{eq:rule-causal-effect}
		\end{align}
	\end{restatable}
	\proofApx
	
	That is, for admissible input $(\descriptionVars, \target, \controls)$, the expression above on the r.h.s. gives us the causal effect of any rule \drule from the rule language \drulespace on \target from observational data.
	Importantly, we have shown that causal rule discovery is a difficult problem in practice---any violation of Def.~\ref{def:admissible-input} would render Eq.~\eqref{eq:rule-causal-effect} non-causal. Having said that, criterion (a) is an implicit assumption in rule discovery, and criterion (b) and (d) are a form of \emph{causal sufficiency}~\cite{scheines:97:intro}, which is a fairly standard assumption in causal inference literature. 
	
	Exceptional cases aside, in practice, we often do not know the complete causal graph. While with some assumptions, we can discover a partially directed graph from observational data~\cite{spirtes:00:book}, a rather pragmatic approach is to leverage domain knowledge to eliminate certain variables following the guidelines in Def.~\ref{def:admissible-input}. For instance, \emph{smoking} causes \emph{tar deposits} in a person’s lungs, therefore both \emph{smoking} and \emph{tar deposits} cannot be in \descriptionVars; this ensures that criterion (c) of Def~\ref{def:admissible-input} is not violated. Moreover, \emph{smoking} may affect a person's \emph{blood pressure}. Thus it is unsafe to include \emph{blood pressure} in \controls---criterion (b) would be violated otherwise. This way, we can get a practical solution that is closer to the truth.

	\subsection{Statistical Considerations}\label{subsec:statistical}
	In practice, we want to estimate \e (Eq.~\eqref{eq:rule-causal-effect}) from a sample drawn from the population. 
	Suppose that we have a sample of \samplesize instances \textbf{stratified} by \controls from the population (or in practice, the sample size is large enough to give relatively accurate estimates of the marginal distribution of \controls). The naive estimator of the causal effect \e is the estimator based on the empirical distribution \hatPr (resp. \empPmf for pmf), i.e. the \textbf{plug-in} estimator:
	\begin{align*}
	\ee
	&= \expectation{\empPmf(\targetval \mid \drule, \controlsval) - \empPmf(\targetval  \mid \bar{\drule}, \controlsval)}\\ 
	&= \sum \limits_{\controlsval \in \controlsspace} \Big ( \empPmf(\targetval \mid \drule, \controlsval) - \empPmf(\targetval \mid \bar{\drule}, \controlsval) \Big ) \empPmf(\controlsval) \\
	&= \sum \limits_{\controlsval \in \controlsspace} ( \hatp - \hatq ) \empPmf(\controlsval) \;,
	\end{align*}
	where $\hatp=\hatPr(\targetval \mid \drule, \controlsval)$, and $\hatq=\hatPr(\targetval \mid \bar{\drule}, \controlsval)$. In a stratified sample, $\empPmf(\controlsval)$ is the same as $\pmf(\controlsval)$.
	As the empirical distribution is a consistent estimator of the population distribution, \ee is a consistent estimator of \e.
	
	The plug-in estimator, however, shows high variance for rules with overly small sample sizes for either of the two events, \drule or $\bar{\drule}$. To illustrate this, in Fig.~\ref{fig:statistical-problem} (left), we show the estimated distribution for the plug-in estimator for a very specific rule of five conditions, and see that while it is close to the true causal effect, it shows very high variance in small samples. This high variance is problematic, as it leads to overfitting: if we use this estimator for the optimisation task over a very large space of rules, the variance will turn into a strong positive bias---we will overestimate the effects of rules from the sample---that dominates the search, and we end up with random results of either extremely specific or extremely general rules.
	
	We address this problem of high variance by biasing the plug-in estimator. In particular, we introduce bias in terms of our confidence in the point estimates using confidence intervals. Note that we need not quantify the confidence of the point estimate $\empPmf(\controlsval)$ as $\empPmf(\controlsval) = \pmf(\controlsval)$; the point estimates of concern are the conditional probabilities \hatp and \hatq. 
	
	In repeated random samples of instances with $\drule\!=\!\true$ and $\controls\!=\!\controlsval$ from the population, the number of instances with \textit{successful} outcome \targetval is a binomial random variable with the success probability $\pmf(\targetval\! \mid\! \drule, \controlsval)$. In a stratum \controlsval of \controls, let \nhatp and \nhatq be the number of instances that satisfy \drule and $\bar{\drule}$, respectively. Then the one-sided binomial confidence interval of \hatp, using a normal approximation of the error distribution, is given by $\zscore \sqrt{\hatp(1-\hatp)/\nhatp}$, where \zscore is the $1-\alpha/2$ quantile of a standard normal distribution for an error rate \erate, or simply the \textbf{z-score} corresponding to the confidence level. For a $95\%$ confidence level, for instance, the error rate is $\alpha\!=\!0.05$, thereby $\zscore\!=\!1.96$. We can easily verify that the maximum value of $\hatp(1-\hatp)$ is $1/4$, and hence the maximum value of the one-sided confidence interval is $\zscore/(2 \sqrt{\nhatp})$.
	Taking a conservative approach, we bias the difference $\hatp - \hatq$ by subtracting the sum of the maximum values of the one-sided confidence intervals of the point estimates, this results in 
	\begin{align*}
	\diff = (\hatp - \hatq ) - \Big ( \zscore/(2 \sqrt{\nhatp}) + \zscore/(2 \sqrt{\nhatq}) \Big ) \; .
	\end{align*}
	
	Note that \diff lower bounds the true probability mass difference in the population with confidence $1-\erate$. That is, there is a $1-\erate$ chance that the true difference is larger than \diff. For a fixed \zscore, the lower bound gets tighter with increasing sample size. In fact, it is easy to see that \diff is a \textbf{consistent} estimator of the true probability mass difference in the population; the introduced bias term vanishes asymptotically. More formally, for a fixed finite \zscore, we have
	\begin{align}
	\lim\limits_{\min(\nhatp, \nhatq) \rightarrow \infty} \zscore/(2 \sqrt{\nhatp}) + \zscore/(2 \sqrt{\nhatq}) = 0 \label{eq:consistency} \; .
	\end{align}
	
	As we deal with empirical probabilities, we can express \diff in terms of counts in a contingency table. Suppose that we have a contingency table as shown in Tab.~\ref{tab:ctable-refinement} (left) for a \controlsval stratum. Then we can express \diff in terms of the cell counts in the contingency table as
	\[
	\diff = \frac{\na}{\nhatp} - \frac{\nc}{\nhatq} - \frac{\zscore}{2 \sqrt{\nhatp}} - \frac{\zscore}{2 \sqrt{\nhatq}} \; .
	\]
	
	
	In the extreme case, however, a rule may select all or none of the instances in a stratum, resulting in $\nhatp\!=\!0$ or $\nhatq\!=\!0$, and hence the empirical conditional probability mass functions can be undefined. In practice, we encounter this problem often, both due to specificity of a rule as well as small sample sizes to begin with.
	
	As a remedy, we apply the Laplace correction to the score. 
	That is, we increment count of each cell in the contingency table by one. This way we start with a uniform distribution within each stratum of \controls. Hence a stratum of size \n increases to $\n+4$, and the total effective sample size increases from \samplesize to $\samplesize + 4 |\controlsspace|$. After applying Laplace correction, we have $\hatPr(\controlsval) = (\n+4)/(\samplesize + 4|\controlsspace|)$, and \diff is given by
	\[
	\diff = \frac{\na+1}{\nhatp+2} - \frac{\nc+1}{\nhatq+2} - \frac{\zscore}{2 \sqrt{\nhatp+2}} - \frac{\zscore}{2 \sqrt{\nhatq+2}} \; .
	\]
	After introducing the bias and applying the Laplace correction to the plug-in estimator, we obtain the \textbf{reliable} estimator of the causal effect as 
	\begin{align}\label{eq:reliable-effect}
	\re = \sum \limits_{\controlsval \in \controlsspace} \diff \empPmf(\controlsval) \; .
	\end{align}
	Note that \re is still a \textbf{consistent} estimator of the causal effect. In contrast to the plug-in estimator, the reliable estimator is much better at generalisation as it avoids overfitting. 
	
	To demonstrate this, let us consider the following example.
	Suppose that we generate the population using the causal graph in Fig.~\ref{fig:bayesian-network}. In addition, we generate five uniformly distributed binary actionable variables $\attribute_2, \attribute_3, \dotsc, \attribute_6$ that are independent of each other as well as the rest of the variables. We can now numerically estimate the variance of the two estimators for a specific rule, e.g. $\drule \equiv \attribute_1\!=\!1 \land \attribute_2\!=\!0 \land \attribute_3\!=\!1 \land \attribute_4\!=\!1 \land \attribute_5\!=\!0\!\land\!\attribute_6\!=\!0$, which does not only contain causal variable $\attribute_1$ but also five actionable variables that are independent of the target \target.

	\begin{figure}[tb]
		\centering
		\ifincludepdf
			\includegraphics{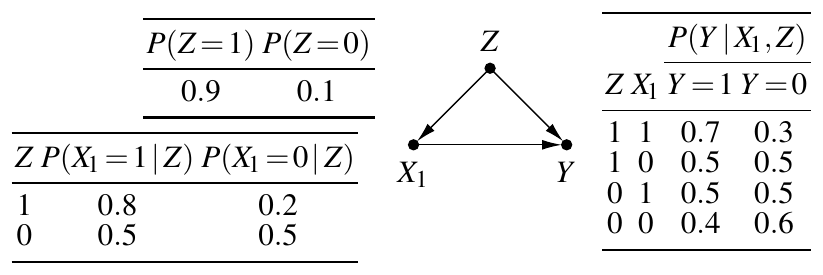}
		\else
			\tikzset{
				events/.style={draw, circle, fill=black, inner sep=0pt, outer sep=0pt, minimum size=1mm},
				link/.style={thin,black,-{Latex[length=2mm, width=1mm]}, shorten >=0mm},
			}
			\setlength{\tabcolsep}{1pt}
			\renewcommand{\arraystretch}{0.8}
			\begin{tikzpicture}[node distance=1cm, font=\small]
			\node [events, label=above:\control] (z) {};
			\node [events, below left = of z, label=below:$\attribute_1$] (x) {};
			\node [events, below right = of z, label=below:\target] (y) {};
			
			
			\draw [link] (z) -- (x);
			\draw [link] (z) -- (y);
			\draw [link] (x) -- (y);
			
			\node [left = 1cm of z] {
				\begin{tabular}{CC}
				\toprule
				\Pr(\control\!=\!1) & \Pr(\control\!=\!0)\\
				\midrule
				0.9 & 0.1\\
				\bottomrule
				\end{tabular}
			};
			
			\node [below left = 5mm and 4mm of x, anchor=east] {
				\begin{tabular}{CCC}
				\toprule
				\control & \Pr(\attribute_{\!1}\!=\!1 \!\mid\! \control) & \Pr(\attribute_{\!1}\!=\!0 \!\mid\! \control)\\
				\midrule
				1 & 0.8 & 0.2\\
				0 & 0.5 & 0.5\\
				\bottomrule
				\end{tabular}
			};
			
			\node [above right = 1mm and 2mm of y, anchor=west] {
				\begin{tabular}{CCCC}
				\toprule
				& & \multicolumn{2}{C}{\Pr(\target \!\mid\! \attribute_{\!1}, \control)}\\
				\cmidrule{3-4}
				\control & \attribute_{\!1} & \target\!=\!1 & \target\!=\!0\\
				\midrule
				1 & 1 & 0.7 & 0.3\\
				1 & 0 & 0.5 & 0.5\\
				0 & 1 & 0.5 & 0.5\\
				0 & 0 & 0.4 & 0.6\\
				\bottomrule
				\end{tabular}
			};
			\end{tikzpicture}
		\fi
		\hfill
		\caption{A causal graph of three variables $\attribute_1$, \target and \control, and alongside their corresponding conditional probabilities as used in the running example. 
		}
		\label{fig:bayesian-network}
	\end{figure}
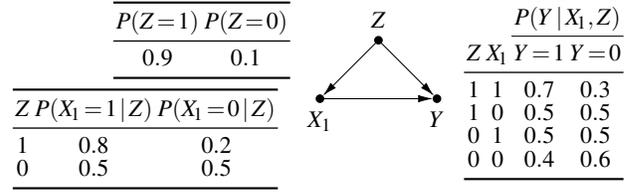
	
	To do so, we draw stratified samples of increasing sizes from the population, and report \ee and \re scores averaged over $25$ simulations along with one sample standard deviation in Fig.~\ref{fig:statistical-problem} (left). We observe that variances of both estimators decrease with increasing sample size. Although the reliable estimator is biased, its variance is relatively low compared to the plug-in estimator. As a result of this low variance, unlike the plug-in estimator, the reliable estimator is indeed able to avoid overfitting, and hence, better at generalisation. Let $\drule^*$ denote the top-$1$ rule in the population, i.e.\ $\drule^* = \argmax_{\drule \in \drulespace} \e$. Let $\varphi^*$ denote the top-$1$ rule using the plug-in estimator, i.e. $\varphi^* = \argmax_{\drule \in \drulespace} \ee$, and $\rho^*$ denote the top-$1$ rule using the reliable estimator, i.e. $\rho^* = \argmax_{\drule \in \drulespace} \re$. In Fig.~\ref{fig:statistical-problem} (right), we plot $\esymb(\varphi^*)$ against $\esymb(\rho^*)$. 
	We observe that with increasing sample sizes $\esymb(\rho^*)$ is both relatively closer, as well as converges much faster to the reference $\esymb(\drule^*)$, which is in agreement with both theory and intuition.
	
	\pgfplotsset{
		legend entry/.initial=,
		every axis plot post/.code={%
			\pgfkeysgetvalue{/pgfplots/legend entry}\tempValue
			\ifx\tempValue\empty
			\pgfkeysalso{/pgfplots/forget plot}%
			\else
			\expandafter\addlegendentry\expandafter{\tempValue}%
			\fi
		},
	}
	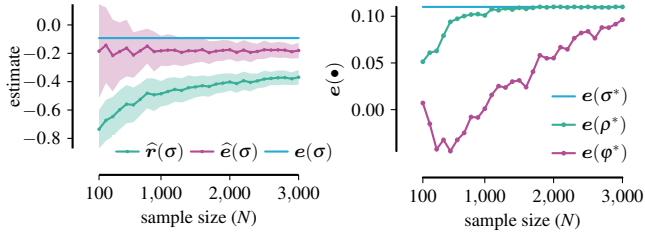
\begin{figure}[tb]
		\centering
		\begin{minipage}{0.5\columnwidth}
			\ifincludepdf
				\includegraphics{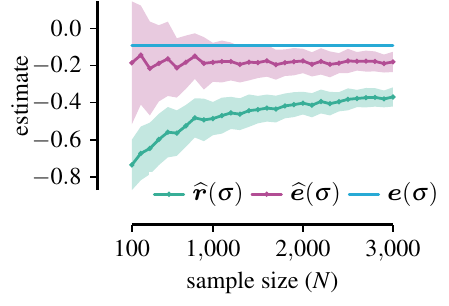}
			\else
			\begin{tikzpicture}
			\begin{axis}[
			thin,
			width=\columnwidth,
			height=3.5cm,
			xlabel=sample size (\samplesize), 
			ylabel=estimate, 
			legend style={draw=none,legend columns=3,at={(0.64,0.1)},anchor=north,line width=0.5pt},
			xmin=100,
			xmax=3000, 
			xtick={100, 1000, 2000, 3000},
			y tick label style={
				/pgf/number format/.cd,
				fixed,
				fixed zerofill,
				precision=1,
				/tikz/.cd
			}, 
			y label style={yshift=30pt},
			x label style={yshift=-15pt},
			clean line,
			mark size=0.25pt,
			every axis plot/.append style={thick}
			]
			\addplot [name path=ru,draw=none] table[x=n, y expr=\thisrow{mean_reliable}+\thisrow{std_reliable}, header=true] {./data/variance.dat};
			\addplot [name path=rl,draw=none] table[x=n, y expr=\thisrow{mean_reliable}-\thisrow{std_reliable}, header=true] {./data/variance.dat};
			\addplot [legend entry=\re, proseco3, mark=*, mark options={fill=proseco3}] table[x=n, y=mean_reliable, header=true] {./data/variance.dat};
			\addplot [fill=proseco3!30] fill between[of=ru and rl];
			
			\addplot [name path=eu,draw=none] table[x=n, y expr=\thisrow{mean_plugin}+\thisrow{std_plugin}, header=true] {./data/variance.dat};
			\addplot [name path=el,draw=none] table[x=n, y expr=\thisrow{mean_plugin}-\thisrow{std_plugin}, header=true] {./data/variance.dat};
			\addplot [legend entry=\ee, proseco2, mark=*, mark options={fill=proseco2}] table[x=n, y=mean_plugin, header=true] {./data/variance.dat};
			\addplot [fill=proseco2!30] fill between[of=eu and el];
			
			\addplot [legend entry=\e, proseco1,mark=none] table[x=n, y=pop_ceff, header=true] {./data/variance.dat}; 
			\end{axis}
			\end{tikzpicture}
			\fi
		\end{minipage}%
		\begin{minipage}{0.5\columnwidth}
			\ifincludepdf
				\includegraphics{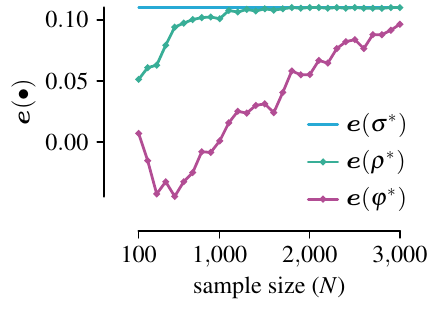}
			\else
			\begin{tikzpicture}
			\begin{axis}[
			clean mark,
			thin,
			mark size=0.5pt,
			cycle list name=gopalette, 
			width=\columnwidth,
			height=3.5cm,
			xlabel=sample size (\samplesize), 
			ylabel=$\esymb(\bullet)$, 
			legend style={draw=none,legend columns=1,at={(0.85,0.5)},anchor=north},
			xmin=100,
			xmax=3000, 
			xtick={100, 1000, 2000, 3000}, 
			y label style={yshift=30pt},
			x label style={yshift=-15pt},
			every axis plot/.append style={thick},
			y tick label style={
				/pgf/number format/.cd,
				fixed,
				fixed zerofill,
				precision=2,
				/tikz/.cd
			},
			]
			
			\addplot [legend entry=$\esymb(\drule^*)$, proseco1] table[x=n, y=e*(.|z), header=true] {./data/generalisation_error.dat};
			\addplot+[legend entry=$\esymb(\rho^*)$, mark=*,proseco3] table[x=n, y=e(r*|z), header=true] {./data/generalisation_error.dat};
			\addplot+[legend entry=$\esymb(\varphi^*)$, mark=*,proseco2] table[x=n, y=e(p*|z), header=true] {./data/generalisation_error.dat};
			\end{axis}
			\end{tikzpicture}
			\fi
		\end{minipage}
		\caption{From the population generated using the causal graph of Fig.~\ref{fig:bayesian-network} together with $5$ additional independent actionable variables $\attribute_2, \dotsc, \attribute_6$, we show (left) variance of the plug-in and reliable estimator of the causal effect for a specific rule that contains variables that are independent of the target, and (right) generalisation error of the effect estimators.}
		\label{fig:statistical-problem}
	\end{figure}

	\section{Discovering Rules}\label{sec:algo}
	Now that we have a reliable and consistent estimator of the causal effect, we turn to discovering rules that maximize this estimator.
	Below, we provide the formal problem definition.
	
	\begin{definition}[Top-$k$ causal rule discovery] Given a sample 
		and a positive integer $k$, find a set $\topkrules \subseteq \drulespace$, $|\topkrules| = k$, such that for all $\drule \in \topkrules$ and $\varphi \in \drulespace \setminus \topkrules$, $\re \geq \rsymb(\varphi)$.
	\end{definition}
	
	Given the hardness of empirical effect maximisation problems~\cite{wang:05:complexity}, it is unlikely that the optimisation of the reliable causal effect allows a worst-case polynomial algorithm. While the exact computational complexity of the causal rule discovery problem is open, here we proceed to develop a practically efficient algorithm using the branch-and-bound paradigm.
	
	\subsection{Branch-and-Bound Search}
	The branch-and-bound search scheme
	finds a solution that optimises the objective function $\func:\searchspace \rightarrow \mathbb{R}$, among a set of admissible solutions \searchspace, also called the search space. Let $\ext(\drule)$, also called the extension of \drule, denote the subset of instances in the sample that satisfy \drule.
	The generic search scheme for a branch-and-bound algorithm requires the following two ingredients:
	\begin{itemize}[label=$\bullet$,leftmargin=*]
		\item A \textbf{refinement operator} $\refineop: \drulespace \rightarrow \powset{\drulespace}$ that is monotone, i.e.\ for $\drule, \varphi \in \drulespace$ with $\varphi = \refineop(\drule)$ it holds that $\ext(\varphi) \subseteq \ext(\drule)$, and that non-redundantly generates the search space \drulespace. That is, for every rule $\drule \in \drulespace$, there is a unique sequence of rules $\drule_0, \drule_1, \dotsc, \drule_\ell = \drule$ with $\drule_i = \refineop(\drule_{i-1})$.
		\item An \textbf{optimistic estimator} $\bound{\func}:\searchspace \rightarrow \mathbb{R}$ that provides an upper bound on the objective function attainable by extending the current rule to more specific rules. That is, it holds that $
		\bound{\func}(\drule) \geq \func(\varphi) \text{ for all } \varphi \in \drulespace \text{ with } \ext(\varphi) \subseteq \ext(\drule)$.
	\end{itemize}
	
	A branch-and-bound algorithm simply enumerates the search space \drulespace starting from the root $\phi$ using the refinement operator \refineop (branch), but based on the optimistic estimator $\bound{\func}$ prunes those branches that cannot yield improvement over the best rules found so far (bound). 
	
	The optimistic estimator depends on the objective function, and 
	there are many optimistic estimators for an objective function \func.
	Not all of these are equally well-suited in practice, as the tightness of the optimistic estimator determines its pruning potential. 
	We consider the \textbf{tight optimistic estimator}~\cite{grosskreutz:08:tightoest} given by
	\begin{align}
	\bound{\func}(\drule) &= \max \{ \func(Q) \mid Q \subseteq \ext(\drule) \}\\
	&\geq \max \{ \func(\varphi) \mid \ext(\varphi) \subseteq \ext(\drule) \text{ for all } \varphi \in \drulespace\} \; .
	\end{align}
	
	The branch-and-bound search scheme also provides an option to trade-off the optimality of the result for the speed. Instead of asking for the \func-optimal result, we can ask for the \approxfactor-approximation result for some approximation factor $\approxfactor \in (0, 1]$. This is done by relaxing the optimistic estimator, i.e.\ $\bound{\func}(\drule) \geq \approxfactor \func(\varphi)$ for all $\varphi \in \drulespace$ with $\ext(\varphi) \subseteq \ext(\drule)$. Lower \approxfactor generally yields better pruning, at the expense of guarantees on the quality of the solution.
	
	In our problem setting, we can define the refinement operator based on the lexicographical ordering of propositions:
	\[
	\refineop(\drule) = \{\drule \wedge \logicaltest_i \mid \logicaltest_i \in \setofprops, i > \max\{j : \logicaltest_j \in \setofprops^{(\drule)} \} \} \; ,
	\]
	where \setofprops is the set of propositions and $\setofprops^{(\drule)}$ is the subset of \setofprops used in \drule. 
	In practice, we need more sophisticated refinement operators in order to avoid the inefficiency resulting from a combinatorial explosion of equivalent rules. This, we can do by defining a closure operator on the rule language (see, e.g. Boley \& Grosskreutz~\cite{boley:09:closure}), which we also employ in our experimental evaluation.
	Next we derive an optimistic estimator for the objective function \rsymb.
	
	\subsection{Efficient optimistic estimator}
	If we look at the definition of \re in Eq.~\eqref{eq:reliable-effect},
	we see that, regardless of \drule, $\empPmf(\controlsval)$ remains the same for a \controlsval stratum. Thus, we can obtain an optimistic estimator of \re by simply bounding \diff for each \controlsval stratum. Let \oestdiff denote the optimistic estimator of \diff. Then the optimistic estimator of \re is given by 
	\[
	\oestrace(\drule) = \sum_{\controlsval \in \controlsspace} \oestdiff \empPmf(\controlsval) \; .
	\]
	To derive the optimistic estimator \oestdiff, for clarity of exposition we first project \diff in terms of free variables \na and \nb, such that we can write
	\[
	\diffsymb(\na, \nb) = \frac{\na\!+\!1}{\na\!+\!\nb\!+\!2} - \frac{\nyone\!-\!\na\!+\!1}{\n\!-\!\na\!-\!\nb\!+\!2} - \frac{0.5\zscore}{ \sqrt{\na\!+\!\nb\!+\!2}} - \frac{0.5\zscore}{\sqrt{\n\!-\!\na\!-\!\nb\!+\!2}} \; .
	\]
	Suppose that we have a contingency table as shown in Tab.~\ref{tab:ctable-refinement} (left) for a \controlsval stratum with the rule \drule. The refinement of \drule, $\drule'=\refineop(\drule)$, results in a contingency table as shown in Tab.~\ref{tab:ctable-refinement} (right). Note that \nyone, \nyzero, and \n do not change within a \controlsval stratum regardless of the rule. Since $\ext(\drule') \subseteq \ext(\drule)$ holds for any $\drule' = \refineop(\drule)$, we have the following relations: $\nap \leq \na$ and $\nbp \leq \nb$.
	
	This implies that the subsets of the extensions of \drule will have contingency table counts \nap in the range $\{0, 1, \dotsc, \na\}$, and \nbp in the range $\{0, 1, \dotsc, b\}$. Let $\countconfig=\{0, 1, \dotsc, \na\} \times \{0, 1, \dotsc, \nb\}$. Then the optimistic estimator of \diff can be defined in terms of \countconfig as
	\begin{align}
	\oestdiff \geq \max \limits_{(\nap, \nbp) \in \countconfig} \diffsymb(\nap, \nbp) \; .
	\end{align}
	
	\begin{table}[tb]
		\centering
		\caption{Contingency tables for (left) a rule \drule, and (right) its refinement $\drule'=\refineop(\drule)$ for a \controlsval stratum of \controls.}
		\renewcommand{\arraystretch}{1.2}
		\begin{minipage}[c]{0.45\columnwidth}
			\centering
			\resizebox{\columnwidth}{!}{%
				\begin{tabular}{rrrr}
					\toprule
					&$\target\!=\!\targetval$&$\target\!\neq\!\targetval$&\\
					\midrule
					$\drule\!=\!\true$&$\na$&$\nb$&\\
					$\drule\!=\!\false$&$\nc$&$\nd$&\\
					\midrule
					$\sum$ & \nyone & \nyzero & \n \\
					\hline
				\end{tabular}
			}
		\end{minipage}
		\hfill
		\begin{minipage}[c]{0.45\columnwidth}
			\centering
			\resizebox{\columnwidth}{!}{%
				\begin{tabular}{rrrr}
					\toprule
					&$\target\!=\!\targetval$&$\target\!\neq\!\targetval$&\\
					\midrule
					$\drule'\!=\!\true$&$\nap$&$\nbp$&\\
					$\drule'\!=\!\false$&$\ncp$&$\ndp$&\\
					\midrule
					$\sum$ & \nyone & \nyzero & \n\\
					\hline
				\end{tabular}
			}
		\end{minipage}
		\label{tab:ctable-refinement}
	\end{table}
	
	The following proposition shows that we can obtain the \textbf{tight optimistic estimate} of \diff in linear time.
	
	\begin{restatable}{proposition}{propositiontightoest}\label{prop:tightoest}
		Let $\countconfig = \{0, 1, \dotsc, \na\} \times \{0, 1, \dotsc, \nb\}$ be the set of all possible configurations of $(\nap, \nbp)$ in Tab.~\ref{tab:ctable-refinement} (right) that can result from refinements of a rule \drule from the contingency table of Tab.~\ref{tab:ctable-refinement} (left). Then the \textbf{tight optimistic estimator} of \diff is given by 
		\begin{align}
		\ttight(\drule, \controlsval) = \max_{\nap \in \{0,1, \dotsc, \na\}} &\frac{\nap+1}{\nap+2} - \frac{\nyone-\nap+1}{\n-\nap+2} - \frac{\zscore}{2 \sqrt{\nap+2}} -\\
		&\phantom{} \frac{\zscore}{2 \sqrt{\n-\nap+2}} .
		\end{align}
		
	\end{restatable}
	\proofApx
	
	\section{Experiments}\label{sec:exp}
	We implemented the branch-and-bound search with priority-queue in the free and open source \realkd\footnote{\realkdlink} Java library, and provide the source code online.\!\footnote{\codelink} 
	All experiments were executed single threaded on  Intel Xeon E5-2643 v3 machine with $256$ GB memory running Linux. We report the results at $\zscore=2.0$, which corresponds to a $95.45\%$ confidence level, and search for optimal top-$k$ rules, i.e. $\approxfactor=1.0$, unless stated otherwise.
	
	\subsection{Performance of the Estimators}\label{subsec:goodness}
	First we evaluate the performance of the proposed estimators of causal effect \e. To this end, we measure the \emph{statistical efficiency} of an estimator by its \textbf{mean squared error (MSE)} as it captures the two most important properties of an estimator: bias and variance. As the optimistic bias is strongest for the best rule and decreases monotonically, we consider the top-1 search here.
	Thus the parameter of interest in the population is the maximum value of the causal effect $\esymb(\drule^*)$, where $\drule^*$ is the maximiser in the population, i.e. $\drule^*= \argmax_{\drule \in \drulespace} \esymb(\drule)$. Using the reliable estimator \rsymb, for instance, we get the reliable effect maximiser $\rho^*$ in the sample, i.e. $\rho^*=\argmax_{\drule \in \drulespace} \re$. As such, $\esymb(\rho^*)$ is our estimate of the estimand $\esymb(\drule^*)$, using \rsymb. Note that $\esymb(\rho^*)$ is a function of the sample, and thus a random variable. Therefore the MSE of $\esymb(\rho^*)$ is given by
	\[
	\mse \left (\esymb(\rho^*) \right ) = \texpectation{\esymb(\rho^*)}{\left ( \esymb(\rho^*) - \esymb(\sigma^*) \right )^2 } \; .
	\]
	Likewise, we can obtain the MSE of $\esymb(\varphi^*)$ using the plug-in estimator \eesymb, where $\varphi^*=\argmax_{\drule \in \drulespace} \ee$.
	
	For this evaluation, first we generate the population using the causal graph in Fig.~\ref{fig:bayesian-network}, and add five \emph{independent} uniformly distributed binary actionable variables $\attribute_2, \attribute_3, \dotsc, \attribute_6$. Then, for a given sample size \samplesize, we sample \samplesize observations from that population, and compute the MSE of the two estimators over $100$ samples. In Fig.~\ref{fig:mse-precision} (left), we show the MSE of the estimators for increasing sample sizes $\samplesize=100, 200, \dotsc, 3000$. As expected, we observe that the MSE decreases for both estimators as the sample size increases. The reliable estimator, however, has a consistently lower MSE than the plug-in estimator. These results show that the proposed reliable estimator is a better choice for optimisation (search) than the naive plug-in estimator.
	
	\subsection{Comparison with the state-of-the-art}
	Next we investigate the quality of rules inferred using the reliable estimator, and compare against other state-of-the-art measures. Although there exists a number of algorithms to infer interesting rules from data, most of them do not provide us \emph{optimal} causal rules. Therefore, in this evaluation, we focus mainly on effect measures they employ, as we can always exhaustively search for optimal rules using those effect measures as long as we keep the rule language small. 
	
	From the exhaustive list of effect measures, we consider the weighted relative accuracy~\cite{flach:99:wracc} for comparison, primarily because it is widely used in inductive rule learners. In addition, we also consider the plug-in estimator without control variables, i.e. \ee with $\controls \coloneqq \emptyset$. In our case, the weighted relative accuracy of the event \drule for an outcome \targetval at the population level is given by
	\[
	\wracc = \pmf(\drule) \Big ( \pmf(\targetval \mid \drule) - \pmf(\targetval) \Big ) \; .
	\]
	In particular, we apply Laplace correction to the plug-in estimators of both the weighted relative accuracy, \ewracc, and $\ee \mid \controls \coloneqq \emptyset$. 
	
	To obtain synthetic data with the known ground truth, we sample observations from the population in our previous evaluation (Sec.~\ref{subsec:goodness}).
	In the causal graph (Fig.~\ref{fig:bayesian-network}), only one actionable variable ($\attribute_1$) affects the target \target; other actionable variables $\attribute_2, \dotsc, \attribute_6$ are independent. As such, only one rule $\drule \equiv \attribute_1=1$ is \emph{relevant}.\!\footnote{The complementary rule $\bar{\drule} \equiv \attribute_1=0$ has a \emph{negative} effect.} We assess an effect measure by evaluating the probability of recovering that single ``true'' rule.
	

	To this end, we take 100 samples, find optimal top-1 rule from each sample, and then calculate the proportion of ``true'' rule among those 100 optimal top-1 rules.
	In Fig.~\ref{fig:mse-precision} (right), we report the probability of recovering the ``true'' rule at various sample sizes.
	We observe that the reliable causal effect is consistently better than other effect measures, and its probability of recovering the core rule exactly approaches $1.0$ rapidly with increasing sample size.
	%
	Other effect measures perform worse particularly when the sample size is small. These results demonstrate that by conditioning on the control variables, reliable causal effect infers relevant causal rules, even on small sample sizes.
	
	
	\begin{figure}[tb]
		\centering
		\begin{minipage}{0.5\columnwidth}
			\ifincludepdf
				\includegraphics{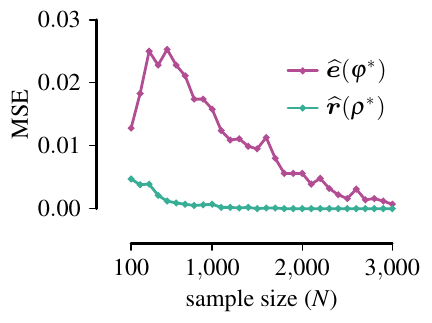}
			\else
			\begin{tikzpicture}
			\begin{axis}[
			clean mark,
			thin,
			mark size=0.5pt,
			cycle list name=gopalette, 
			width=\columnwidth,
			height=3.5cm,
			xlabel=sample size (\samplesize), 
			ylabel=MSE, 
			legend style={draw=none,legend columns=1,at={(0.8,0.85)},anchor=north},
			xmin=100,
			xmax=3000, 
			xtick={100, 1000, 2000, 3000},
			ymin=0.0,
			ymax=0.03,
			y label style={yshift=30pt},
			x label style={yshift=-15pt},
			every axis plot/.append style={thick},
			y tick label style={
				/pgf/number format/.cd,
				fixed,
				fixed zerofill,
				precision=2,
				/tikz/.cd
			},
			]
			
			\addplot [legend entry=$\eesymb(\varphi^*)$, mark=*, proseco2] table[x=n, y=MSE(e), header=true] {./data/mse.dat};
			\addplot+[legend entry=$\rsymb(\rho^*)$, mark=*, proseco3] table[x=n, y=MSE(r), header=true] {./data/mse.dat};
			\end{axis}
			\end{tikzpicture}
			\fi
		\end{minipage}%
		\begin{minipage}{0.5\columnwidth}
			\ifincludepdf
				\includegraphics{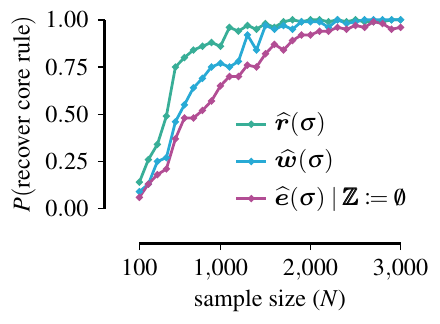}
			\else
				\begin{tikzpicture}
			\begin{axis}[
			thick,
			width=\columnwidth,
			height=3.5cm,
			xlabel=sample size (\samplesize), 
			ylabel={$\Pr($recover core rule$)$}, 
			legend style={draw=none,legend columns=1,anchor=west, at={(0.35, 0.25)}, line width=0.5pt},
			xmin=100,
			xmax=3000, 
			ymin=0,
			ymax=1.0,
			ytick={0, 0.25, ..., 1.0},
			xtick={100, 1000, 2000, 3000}, 
			y tick label style={
				/pgf/number format/.cd,
				fixed,
				fixed zerofill,
				precision=2,
				/tikz/.cd
			},
			y label style={yshift=30pt},
			x label style={yshift=-15pt},
			clean errorbar,
			every axis plot/.append style={thick}
			]
			\addplot [legend entry=\re, proseco3, mark=*, mark options={fill=proseco3}] table[x=n, y=mean_p1_rceff, header=true] {./data/precision_at_k.dat};
			
			\addplot [legend entry=\ewracc, proseco1, mark=*, mark options={fill=proseco1}] table[x=n, y=mean_p1_wracc, header=true] {./data/precision_at_k.dat};
			
			\addplot [legend entry=$\ee \mid \controls \coloneqq \emptyset$, proseco2, mark=*, mark options={fill=proseco2}] table[x=n, y=mean_p1_eff, header=true] {./data/precision_at_k.dat};
			
			\end{axis}
			\end{tikzpicture}
			\fi
		\end{minipage}
		\caption{(left) Mean squared error (MSE) of the plug-in estimator \eesymb and the reliable estimator \rsymb of the population optimal causal effect $\esymb(\drule^*)$. (right) The probability of recovering the core rule from
			100 samples, for the Laplace-corrected plug-in estimator of causal effect, \ee, with an empty \controls, the Laplace corrected plug-in estimator of weighted relative accuracy, \ewracc, and the reliable estimator of causal effect, \re.}
		\label{fig:mse-precision}
	\end{figure}
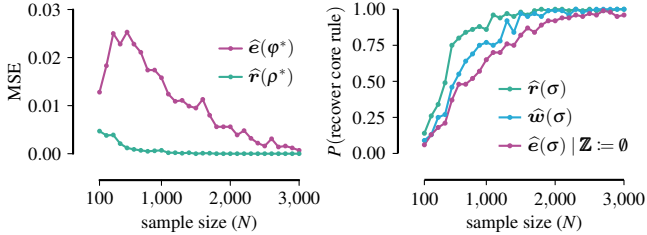
	
	\subsection{Qualitative Study on Real-World Data}
	Next we investigate whether rules discovered by reliable causal effect are meaningful. 
	To this end, we consider the \emph{titanic} training set from the Kaggle repository.\!\footnote{\url{https://www.kaggle.com/c/titanic}} 
	The sinking of RMS Titanic is one of the most notorious shipwrecks in history. One of the reasons behind such tragic loss of lives was the lack of lifeboats. During the evacuation, some passengers were treated differently than the others; some groups of people were, hence, more likely to survive than the others. Thus, it is of interest to find the conditions that have causal effect on the \emph{survival} (\target). The dataset contains demographics and travel attributes of the passengers.
	
	Existing causal discovery methods are not applicable as we have a mixed data set. We also do not know the complete causal graph. Therefore we take a pragmatic approach using domain knowledge. If we could perform a hypothetical intervention of changing the sex of a person, this will also change their title, but not the other way around. Thus it is reasonable to assume that \emph{sex} causes \emph{title}. As putting them together in \descriptionVars would violate criterion (c) of Def.~\ref{def:admissible-input}, we only keep one of them, i.e. \emph{sex}, in \descriptionVars. Similarly we can argue that fare causes passenger class. Therefore we only keep \emph{class} in \descriptionVars. Overall, none of the variables seem to confound (co-cause) \target and other variables. Altogether, we have 
	\begin{align}
	\controls &\coloneqq \emptyset\\
	\target &\coloneqq \text{\emph{survived}}\\
	\descriptionVars &\coloneqq \{\text{\emph{class}}, \text{\emph{pname}}, \text{\emph{sex}}, \text{\emph{age}}, \text{\emph{sib\_sip}}, \text{\emph{par\_ch}}, \text{\emph{embarked}} \}
	\end{align}
	
	In Tab.~\ref{res:titanic}, we present optimal top-3 causal rules discovered from the input above using the proposal method. The coverage of a rule is a fraction of instances that belong to its extension, i.e. $\coverage(\drule) = |\ext(\drule)|/\samplesize$. 
	
	We observe that being a female passenger from the first, or the second class has the highest effect on survival with a reliable causal effect estimate of $\rsymb(\drule_1) = 0.576$. It is well-known that passengers from different classes were treated differently during evacuation. What is interesting is that although females were more likely to survive, this only applied to the females from the first and the second class; this is also corroborated by the fact that roughly half of the females from the third class did not survive the mishap compared to the only one-tenth from the other two classes combined. 
	
	The other two rules corroborate the adage of ``women and children'' first. The fact that all those rules came out on top with mere $<20\%$ coverage shows that reliable causal effect can discover rare rules.
	
	\begin{table}[tb]
		\centering
		\caption{Top-3 causal rules discovered on the titanic dataset with ``survival'' as a target variable. 
		}
		\label{res:titanic}
		\renewcommand{\arraystretch}{1.2}
		\resizebox{\columnwidth}{!}{%
			\begin{tabular}{p{6cm} R@{\hskip 2mm} R@{\hskip 2mm}}
				\toprule
				Top-$3$ rules (\drule) & \coverage(\drule) & \re\\
				\midrule
				class $\leq 2$ \andl sex = female & 0.1907 & 0.576\\
				class $\leq 2$ \andl sex = female \andl par-ch $\leq 2$ & 0.1885 & 0.573\\
				class $\leq 2$ \andl sex = female \andl sib-sp $\leq 2$ & 0.1874 & 0.572\\
				\bottomrule
			\end{tabular}
		}
	\end{table}

	\section{Discussion}\label{sec:disc}	
	The main focus of this discussion are the assumptions (in Def.~\ref{def:admissible-input}) required for causal rule discovery and their practical implications. First we note that it is \emph{impossible} to do causal inference from observational data without making assumptions, as the joint distribution alone cannot tell us what happens when the system undergoes changes through interventions~\cite{pearl:09:book}. Through causal diagrams, we make such assumptions more explicit and transparent. Often, the more explicit the assumption, the more criticism it invites. Explicit assumptions, however, can be good as they provide a way to verify our models, and improve them. Def.~\ref{def:admissible-input}, for instance, provides guidelines for variable selection process for causal rule discovery.
	
	For the inferred rules to be causal, the input must be admissible. Although criterion (a), (b) and (d) are fairly standard in the literature, criterion (c) is new and specific to rule discovery. It requires that there are no edges between actionable variables. Over a large group of actionable variables, this can be a strong assumption. The naive way to remove this assumption would be to include rest of the actionable variables $\descriptionVars \setminus \attributes$ in the set of control variables to block any spurious path between actionable variables \attributes in a rule \drule and the target \target via $\descriptionVars \setminus \attributes$. By doing so, however, we may not only violate other criteria, but the search also gets complicated. 
	
	On the statistical side, a direct correction for controlling the familywise error rate of confidence intervals would not lead to an effective approach to discover causal effects. Therefore, we followed the statistical learning approach and designed an estimator with small generalisation error. This use of confidence intervals is reminiscent of, e.g., upper confidence bound strategies in multi-armed bandit problems (yielding an optimal policy despite not controlling the familywise error rate of reward estimates).
	
	\section{Conclusion}
	Traditional descriptive rule discovery techniques do not suffice for discovering reliable causal rules from observational data. Among the sources of inconsistency we have that observational effect sizes are often skewed by the presence of confounding factors. Second, naive empirical effect estimators have a high variance, and, hence, their maximisation is highly optimistically biased unless the search is artificially restricted to high frequency events. In this work, we presented a causal rule discovery approach that addresses both these issues. We measured the causal effect of a rule from observational data by adjusting for the effect of potential confounders. In particular, we gave the graphical criteria under which causal rule discovery is possible. To discover reliable causal rules from a sample, we proposed a conservative and consistent estimator of the causal effect, and derived an efficient and exact algorithm based on branch-and-bound search that maximises the estimator. The proposed algorithm is efficient and finds meaningful rules. 

	\bibliographystyle{abbrv}
	\bibliography{paper}
	
	\appendix
	
	\section*{Proof of Proposition~\ref{prop:backdoor}}
	\propbackdoor*
	\begin{proof}
		We prove this proposition graphically. Recall the criteria for the input $(\actionables, \target, \controls)$ to be admissible.
		\begin{enumerate}[leftmargin=*, label=(\alph*)]
			\setlength\itemsep{0em}
			\item there are no outgoing edges from \target to any \attribute in \actionables,
			\item no outgoing edges from any \attribute in \actionables to any \control in \controls,
			\item no edges between actionable variables \actionables, and
			\item no edges between any unobserved \latent and \attribute in \actionables.
		\end{enumerate}
		Observe that any spurious path between any \attribute in \attributes and \target can be formed through one of the following ways:
		\begin{itemize}[leftmargin=*]
			\setlength\itemsep{0em}
			\item from \target to \attributes directly,
			\item via control variables \controls,
			\item via other actionable variables $\descriptionVars \setminus \attributes$,
			\item via latent variables \latents.
		\end{itemize}
		Criterion (a) rules out trivial spurious paths from \target to $\descriptionVars \supseteq \attributes$ that cannot be blocked by any \controls. Criterion (b) ensures that any spurious path unblocked by one control variable is blocked by another. Criterion (c) ensures that there are no spurious paths between any subset of actionable variables \attributes in the rule \drule and \target via other actionable variables $\descriptionVars \setminus \attributes$. To see this, suppose that we have two actionable variables $\attribute_1$ and $\attribute_2$, and a rule $\drule \equiv \attribute_1=1$. If the causal graph contains the path $\attribute_1 \leftarrow \attribute_2 \rightarrow \target$, we will have a biased estimate of the causal effect. Criterion (d) is really just a form of standard causal sufficiency~\citep{scheines:97:intro}. As there are no edges between any latent variable \latent in \latents and any \attribute in \descriptionVars, by conditioning on \controls, we block any spurious path between any \attributes and \target via \latents. Thus, if the input is admissible, the control variables \controls block all spurious paths between any subset \attributes of \descriptionVars and \target.
	\end{proof}
	
	\section*{Proof of Theorem~\ref{theorem:causaleffect}}
	\theoremcausaleffect*
	\begin{proof}
		Recall that $\pmf(\targetval \mid \doo(\policy_\drule))$ is given by
		\begin{align}
		\pmf(\targetval \mid \doo(\policy_\drule)) &=\sum\limits_{\controlsvec \in \controlsspace} \pmf(\controlsval) \sum\limits_{\drule(\actionablesVal)=\true} \pmf(\targetval \mid \attributesval, \controlsval) \policy_\drule(\doo(\actionablesVal))
		\intertext{Using the stochastic policy, $\pmf(\targetval \mid \doo(\policy_\drule))$ reduces to}
		&=\sum\limits_{\controlsvec \in \controlsspace} \pmf(\controlsval) \sum\limits_{\substack{\drule(\actionablesVal)=\true}} \pmf(\targetval \mid \attributesval, \controlsval) \pmf(\attributesval \mid \drule, \controlsval)\\
		&= \sum\limits_{\controlsvec \in \controlsspace} \pmf(\controlsval) \sum\limits_{\drule(\actionablesVal)=\true} \pmf(\targetval \mid \attributesval, \controlsval) \frac{\pmf(\attributesval, \drule \mid \controlsval)}{\pmf(\drule \mid \controlsval)}
		\intertext{for $\drule(\actionablesVal)=\true$, it holds that $\pmf(\attributesval, \drule \mid \controlsval)=\pmf(\attributesval \mid \controlsval)$; thus}
		\pmf(\targetval \mid \doo(\policy_\drule)) &= \sum\limits_{\controlsvec \in \controlsspace} \pmf(\controlsval) \sum\limits_{\drule(\actionablesVal)=\true} \pmf(\targetval \mid \attributesval, \controlsval) \frac{\pmf(\attributesval \mid \controlsval)}{\pmf(\drule \mid \controlsval)}\\
		&=\sum\limits_{\controlsvec \in \controlsspace} \frac{\pmf(\controlsval)}{\pmf(\drule \mid \controlsval)} \sum\limits_{\drule(\actionablesVal)=\true} \pmf(\targetval, \attributesval \mid \controlsval)
		\intertext{
			since $\sum_{\drule(\attributesval)=\true} \Pr(\targetval, \attributesval \mid \controlsval) = \Pr(\targetval, \drule \mid \controlsval)$, this results in
		}
		\pmf(\targetval \mid \doo(\policy_\drule)) &=\sum\limits_{\controlsvec \in \controlsspace} \frac{\pmf(\controlsval)}{\pmf(\drule \mid \controlsval)}  \pmf(\targetval, \drule \mid \controlsval)\\
		&=\sum\limits_{\controlsvec \in \controlsspace} \pmf(\controlsval) \pmf(\targetval \mid \drule, \controlsval)\\
		&= \expectation{\pmf(\targetval \mid \drule, \controls)} \; .
		\end{align}
		Substituting the above in the definition of $\ec$, we get
		\begin{align}
		\ec &=  \pmf(\targetval \mid \doo(\policy_\drule)) - \pmf(\targetval \mid \doo(\policy_{\bar{\drule}}))\\
		&= \expectation{\pmf(\targetval \mid \drule, \controls)} - \expectation{\pmf(\targetval \mid \bar{\drule}, \controls)} \; .
		\end{align}
	\end{proof}

	\section*{Proof of Proposition~\ref{prop:tightoest}}
	\propositiontightoest*
	\begin{proof}
		The expression for $\diffsymb(\nap, \nbp)$ from the contingency table in Tab.~\ref{tab:ctable-refinement} (right) is given by
		\begin{align}
		\diffsymb(\nap, \nbp) &= \frac{\nap+1}{\nap+\nbp+2} - \frac{\nyone-\nap+1}{\n-\nap-\nbp+2} - \frac{\zscore}{2 \sqrt{\nap+\nbp+2}} -\\ &\phantom{test} \frac{\zscore}{2 \sqrt{\n-\nap-\nbp+2}} .
		\end{align}
		Combining the first and the third term above, we get
		\begin{align}
		\lambdaz(\nap, \nbp) &= \frac{2\nap + 2 - \zscore \sqrt{\nap+\nbp+2}}{2(\nap+\nbp+2)} - \frac{\nyone-\nap+1}{\n-\nap-\nbp+2} -\\
		&\phantom{test} \frac{\zscore}{2 \sqrt{\n-\nap-\nbp+2}} . 
		\end{align}
		Note that if we fix the value of \nap, then the value of \nbp that maximises $\diffsymb(\nap, \nbp)$ has to maximise the first term above, but minimise the other two terms. Observe that $\nbp=0$, out of $\nbp \in \{0, 1, \dotsc, b\}$, does both simultaneously. Thus we have the following relation: $\diffsymb(\nap, 0) > \diffsymb(\nap, \nbp) \text{ for all } \nbp > 0$.
		
		The tight optimistic estimator of \diff is then the maximum value over all possible configurations \countconfig, i.e.
		\begin{align}
		\ttight(\drule, \controlvec) &= \max_{\nap \in \{0,1, \dotsc, a\}} \diffsymb(\nap, 0)\\
		&= \max_{\nap \in \{0,1, \dotsc, a\}} \frac{\nap+1}{\nap+2} - \frac{\nyone-\nap+1}{\n-\nap+2} - \frac{\zscore}{2 \sqrt{\nap+2}} -\\&\phantom{\max_{\nap \in \{0,1, \dotsc, a\}}} \frac{\zscore}{2 \sqrt{\n-\nap+2}} .
		\end{align}
	\end{proof}
	
	\section*{Effect of \zscore}
	The reliable estimator \rsymb of causal effect has a user-defined parameter \zscore that represents our confidence in the point estimate. It is easy to see that a 0\% confidence level corresponds to the plug-in estimator \eesymb, as then we would have a z-score of $\zscore=0$. As we increase the confidence level, \re gets more conservative. How does this affect the performance of the reliable estimator? 
	
	To answer this question, we sample observations from the population in our previous evaluation (Sec.~\ref{subsec:goodness}). In Fig.~\ref{fig:mse_vs_beta}, we plot the mean squared error (MSE) of $\esymb(\rho^*)$ that uses the reliable estimator \rsymb at various confidence levels: $50\%, 60\%, 70\%, 80\%, 90\%, 99\%$. We observe that $\mse \left (\esymb(\rho^*) \right)$ decreases with increasing sample sizes at all confidence levels. The MSE is much better at higher confidence levels particularly in the beginning when sample sizes are small. These results suggest that higher confidence levels lead to more reliable rules, in terms of their closeness to the maximum true effect. 
	It might be tempting then to go for a 100\% confidence level. However, 100\% confidence is only achievable with an infinite sample size. Moreover, at a 100\% confidence level, we have $\zscore=\infty$, and in turn $\re=\infty$.
	This suggests that we can calibrate the optimal \zscore for a given sample size somewhere below the 100\% confidence level.
	
	\begin{figure}[tb]
		\ifincludepdf
			\includegraphics{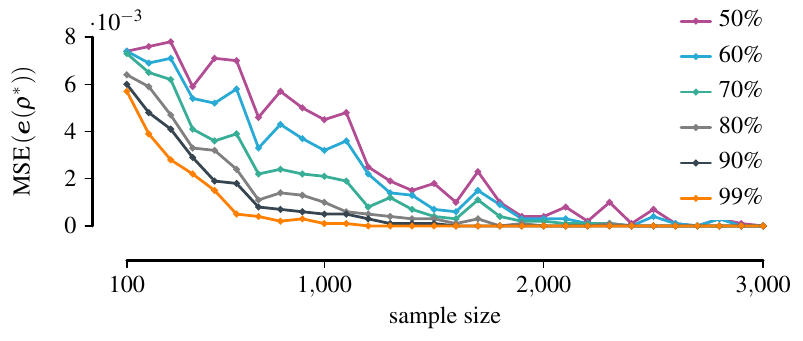}
		\else
			\begin{tikzpicture}
			\begin{axis}[
			clean mark,
			thin,
			mark size=0.5pt,
			cycle list name=gopalette, 
			width=0.95\columnwidth,
			height=3.5cm,
			ylabel=$\mse \left (\esymb(\rho^*) \right)$,
			xlabel=sample size, 
			legend style={draw=none, legend columns=1, at={(0.94,1.2)},anchor=north},
			ymax=0.008,
			xtick={100, 1000, 2000, 3000},
			y label style={yshift=32pt},
			x label style={yshift=-15pt},
			every axis plot/.append style={thick},
			scaled y ticks=base 10:3, 
			y tick scale label style={xshift=-15pt}
			]
			\addplot [legend entry={50\%}, mark=*, proseco2] table[x=n, y=b0.67, header=true] {./data/mse_vs_beta.dat};
			\addplot [legend entry={60\%}, mark=*, proseco1] table[x=n, y=b0.84, header=true] {./data/mse_vs_beta.dat};
			\addplot [legend entry={70\%}, mark=*, proseco3] table[x=n, y=b1.04, header=true] {./data/mse_vs_beta.dat};
			\addplot [legend entry={80\%}, mark=*, gray] table[x=n, y=b1.28, header=true] {./data/mse_vs_beta.dat};
			\addplot [legend entry={90\%}, mark=*, charcoal] table[x=n, y=b1.64, header=true] {./data/mse_vs_beta.dat};
			\addplot [legend entry={99\%}, mark=*, orange] table[x=n, y=b2.58, header=true] {./data/mse_vs_beta.dat};
			\end{axis}
			\end{tikzpicture}
		\fi
		\caption{Mean squared error of the reliable estimator \rsymb of the population optimal causal effect $\esymb(\drule^*)$ at various \zscore.}\label{fig:mse_vs_beta}
	\end{figure}
	
	\section*{Efficiency of the Branch-and-Bound Search}
	Next we assess efficiency of the branch-and-bound search. To this end, first we search for top-$1$ rule in all the standard classification datasets from the KEEL repository.\!\footnote{\url{https://sci2s.ugr.es/keel/datasets.php}}
	The diversity of these datasets in terms of their sample size and number of actionable variables provides a reasonable picture on the efficiency of the proposed search algorithm in a real-world scenario.
	For each dataset, we select the classification target as the target, and randomly select one of the attributes as the control variable. As outcome \targetval, we select one of the outcomes of the target \target. We discretise a real-valued actionable variable into maximum $8$ equi-frequent bins.
	
	\begin{table}[tb]
		\centering
		\caption{Summary of the datasets used for the empirical evaluation along with the efficiency results. For each dataset, we report the chosen target variable (\target), the chosen control variables (\controls), the sample size (\samplesize), the number of actionable variables ($|\descriptionVars|$), the approximation factor (\approxfactor), the runtime in seconds, and the number of nodes expanded during search.}
		\label{tab:efficiency-results}
		\renewcommand{\arraystretch}{1.5}
		\resizebox{\columnwidth}{!}{%
			\begin{tabular}{@{}lllRRRRR@{}}
				\toprule
				Dataset & Target (\target) & Control (\controls) & \samplesize & |\descriptionVars| & \approxfactor & \text{time (s)} & \#\text{nodes}\\ 
				\midrule
				adult & class & sex & 48,842 & 13 & 0.8 & 1,717 & 258,575\\
				australian & class & a4 & 690 & 13 & 1.0 & 146 & 952,175\\
				automobile & output & engine-type & 205 & 24 & 1.0 & 1 & 15,167\\
				breast & class & age & 286 & 8 & 1.0 & 78 & 420\\
				car & acceptability & safety & 1,728 & 5 & 1.0 & 0.02 & 33\\
				chess & class & bkblk & 3,196 & 35 & 1.0 & 851 & 1,613,398\\
				connect-4 & class & a1 & 67,557 & 61 & 0.3 & 1,679 & 140,707\\
				crx & class & a1 & 690 & 14 & 1.0 & 14 & 101,621\\
				fars & injury-severity & case-state & 100,968 & 28 & 0.8 & 724 & 22,328\\
				flare & class & prev24hour & 1,066 & 10 & 1.0 & 0.014 & 32\\
				german & customer & statusAndSex & 1,000 & 19 & 1.0 & 8 & 43,007\\
				housevotes & class & el-salvador-aid & 435 & 15 & 1.0 & 0.007 & 57\\
				kddcup & class & atr-6 & 494,020 & 40 & 0.99 & 37 & 219\\
				kr-vs-k & game & white-king-col& 28,056 & 5 & 1.0 & 30 & 7,304\\
				lymphography & classes & changes-in-lym & 148 & 17 & 1.0 & 0.14 & 1,666\\
				mushroom & class & gill-size & 8,124 & 21 & 1.0 & 0.307 & 215\\
				nursery & class & social & 12,690 & 7 & 1.0 & 0.66 & 279\\
				post-operative & decision & l-core & 90 & 7 & 1.0 & 0.016 & 258\\
				splice & class & pos1 & 3,190 & 59 & 1.0 & 1.03 & 1,855\\
				tic-tac-toe & class & topleft & 958 & 8 & 1.0 & 0.11 & 488\\
				titanic & survived & sex & 891 & 9 & 1.0 & 4.5 & 26,700\\
				zoo & type & aquatic & 101 & 15 & 1.0 & 0.009 & 96\\
				\bottomrule
			\end{tabular}
		}
	\end{table}
	
	In the Tab.~\ref{tab:efficiency-results}, we provide the summary of the datasets along with the efficiency results.
	For each dataset, we report the target (\target), the set of control variables (\controls),\!\footnote{Although we use only one variable in \controls, it does not really work to our advantage here. Assuming that observations are roughly uniformly distributed across values of \controls, we only have $\samplesize/|\controlsspace|$ observations within each group of \controls. Thus group sizes gets smaller as we we add more variables in \controls. This in turn increases the bias term in the reliable estimator, thereby decreasing the number of rules with \emph{positive} reliable causal effect. That fact combined with the tightness of the optimistic estimator eventually speeds up the branch-and-bound search as we add more variables to \controls.} the sample size (\samplesize), the number of actionable variables ($|\descriptionVars|$), the approximation factor (\approxfactor) such that the branch-and-bound search finishes within an hour, the runtime in seconds, and the number of nodes expanded during the search. We observe that the branch-and-bound search with the tight optimistic estimator retrieves the \emph{optimal} top-$1$ result within seconds for most datasets, taking up to roughly an hour for a few datasets. 
	
	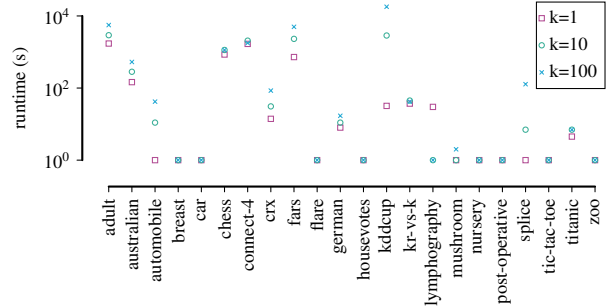
\begin{figure}[tb]
		\ifincludepdf
			\includegraphics{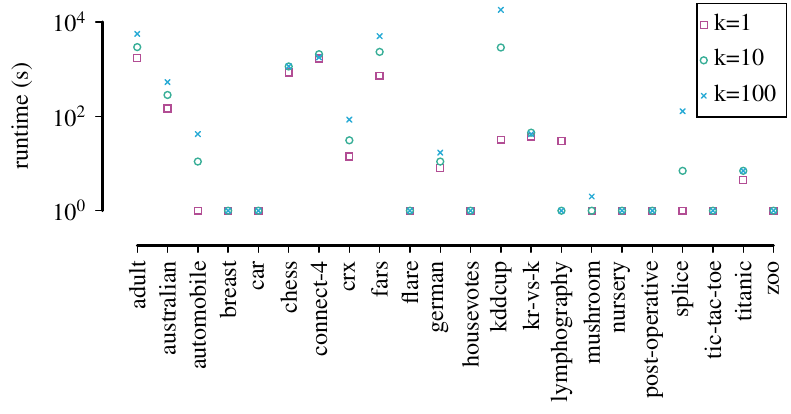}
		\else
			\begin{tikzpicture}
			\begin{axis}[
			clean mark,
			only marks,
			thin,
			mark size=1pt,
			cycle list name=gopalette, 
			width=0.95\columnwidth,
			height=3.5cm,
			ylabel=runtime (s), 
			legend style={legend columns=1,at={(0.95,1.1)},anchor=north},
			symbolic x coords={adult,australian,automobile,breast,car,chess,connect-4,crx,fars,flare,german,housevotes,kddcup,kr-vs-k,lymphography,mushroom,nursery,post-operative,splice,tic-tac-toe,titanic,zoo},
			xtick=data,
			ymin=1,
			ymax=10000,
			y label style={yshift=36pt},
			x label style={yshift=-15pt},
			x tick label style={rotate=90},
			every axis plot/.append style={very thick},
			ymode = log
			]
			\addplot [legend entry={k=1}, mark=square, every mark/.append style={thin, solid, fill=none}, proseco2] table[x=dataset, y=k1, header=true, col sep=comma] {./data/timing_comparison.dat};
			\addplot [legend entry={k=10}, mark=o, every mark/.append style={thin, solid, fill=none}, proseco3] table[x=dataset, y=k10, header=true, col sep=comma] {./data/timing_comparison.dat};
			\addplot [legend entry={k=100}, mark=x, every mark/.append style={thin, solid, fill=none}, proseco1] table[x=dataset, y=k100, header=true, col sep=comma] {./data/timing_comparison.dat};
			\end{axis}
			\end{tikzpicture}
		\fi
		\caption{Runtime, in seconds, of the top-$k$ branch-and-bound search algorithm for different values of $k$ for the real-world datasets in Tab.~\ref{tab:efficiency-results}.}\label{fig:scalability}
	\end{figure}
	
	In practice, it may be of interest to look for multiple rules for various reasons. Therefore, next we evaluate the scalability of the branch-and-bound algorithm with respect to $k$ in top-$k$ search. In Fig.~\ref{fig:scalability}, we show the runtime of the branch-and-bound algorithm for $k=1, 10, 100$ in all the standard classification datasets from the KEEL repository, using same approximation factors. For most datasets, we observe that the branch-and-bound search finishes within seconds at all values of $k$. For a few datasets, even though runtime of the algorithm increases with increasing value of $k$, it finishes within a couple of hours.
\end{document}